\documentclass[oneside,english]{amsart}
\usepackage[LGR,T1]{fontenc}
\usepackage[latin9]{inputenc}
\usepackage{babel}
\usepackage{textcomp}
\usepackage{amsthm}
\usepackage{amstext}
\usepackage{amssymb}
\usepackage{graphicx}
\usepackage{esint}
\usepackage[unicode=true]
 {hyperref}

\makeatletter

\DeclareRobustCommand{\greektext}{%
  \fontencoding{LGR}\selectfont\def\encodingdefault{LGR}}
\DeclareRobustCommand{\textgreek}[1]{\leavevmode{\greektext #1}}
\DeclareFontEncoding{LGR}{}{}
\DeclareTextSymbol{\~}{LGR}{126}

\numberwithin{equation}{section}
\theoremstyle{plain}
\newtheorem{thm}{\protect\theoremname}
  \theoremstyle{plain}
  \newtheorem{lem}[thm]{\protect\lemmaname}
  \theoremstyle{plain}
  \newtheorem{prop}[thm]{\protect\propositionname}

\makeatother

  \providecommand{\lemmaname}{Lemma}
  \providecommand{\propositionname}{Proposition}
\providecommand{\theoremname}{Theorem}

\begin{document}

\title{Nonparametric Testing for Heterogeneous Correlation}

\author{Stephen Bamattre, Rex Hu and Joseph S. Verducci}

\begin{abstract}
\noindent In the presence of weak overall correlation, it may be useful
to investigate if the correlation is significantly and substantially
more pronounced over a subpopulation. Two different testing procedures
are compared. Both are based on the rankings of the values of two
variables from a data set with a large number n of observations. The
first maintains its level against Gaussian copulas; the second adapts
to general alternatives in the sense that that the number of parameters
used in the test grows with $n$. An analysis of wine quality illustrates
how the methods detect heterogeneity of association between chemical
properties of the wine, which are attributable to a mix of different
cultivars. 
\end{abstract}

\maketitle

\section*{Introduction}

The goal of this paper is to offer new methods for discovering association
between two variables that is supported only in a subpopulation. For
example, while higher counts of HDLs are generally associated with
lower risk of myocardial infarction, researchers (Voight et al., 2012;
Katz, 2014) have found subpopulations that do not adhere to this trend.
In marketing, subpopulations of designated marketing areas (DMAs)
in the US respond differentially to TV advertising campaigns, and
the identification of DMAs that are sensitive to ad exposure enables
efficient spending of ad dollars. In preclinical screening of potential
drugs, various subpopulations of chemicals elicit concomitant responses
from sets of hepatocyte genes, which can be used to discover gene
networks that breakdown classes of drugs, without having to pre-specify
how the classes are formed. The new methods thus lead to a whole new
approach to analysis of large data sets. 

When covariates are available, regression analysis classically attempts
to identify a supporting subpopulation via interaction effects, but
these may be difficult to interpret properly. In the presence of overall
correlation, it may be useful to investigate directly if the correlation
is significantly and substantially more pronounced over a subpopulation.
This becomes feasible when representatives of supporting subpopulations
are embedded in large samples. The novel statistical tests described
in this paper are designed to probe large samples to ascertain if
there is such a subpopulation. 

The general setting is this: A large number \emph{n} of observations
are sampled from a bivariate continuous distribution. The basic assumption
is that the population consists of two subpopulations. In one, the
two variables are positively (or negatively) associated; in the other,
the two variables are independent. While some distributional assumptions
are required even to define the notion of homogeneous association,
the underlying intent is to make the tests robust to assumptions about
the distributions governing both the null and alternative hypotheses. 

Notation for the rest of the paper is as follows: Let $X\sim F$ and
$Y\sim G$ have joint, continuous distribution H. For any sample $\left\{ \left(x_{i},y_{i}\right)\mid i=1,\ldots,n\right\} $,
the empirical marginal distributions are defined by
\begin{eqnarray*}
\hat{F}_{n}\left(x\right)=\frac{1}{n}\sum_{i=1}^{n}1\left\{ x_{i}\leq x\right\}  & \textrm{and} & \hat{G}_{n}\left(y\right)=\frac{1}{n}\sum_{i=1}^{n}1\left\{ y_{i}\leq y\right\} .
\end{eqnarray*}
The ranking $\pi$ of the sample $\left\{ x_{i}\mid i=1,\ldots,n\right\} $
is the function\\
 $\pi:\left\{ x_{i}\mid i=1,\ldots,n\right\} \rightarrow\left\{ 1,\ldots,n\right\} $
defined by
\[
\pi\left(x_{i}\right)=\sum_{j=1}^{n}1\left\{ x_{i}\leq x_{j}\right\} .
\]
The corresponding ranking of $\left\{ y_{i}\mid i=1,\ldots,n\right\} $
is denoted by $\nu$. Spearman\textquoteright s footrule distance
with a sample $\left\{ \left(x_{i},y_{i}\right)\mid i=1,\ldots,n\right\} $
is defined through the sample rankings as
\[
d_{S}=\sum_{i=1}^{n}\left|\pi\left(x_{i}\right)-\nu\left(y_{i}\right)\right|.
\]
The Kendall distance associated with the sample is defined as
\begin{eqnarray*}
D_{K}\left(\left[x_{1},\ldots,x_{n}\right],\left[y_{1},\ldots,y_{n}\right]\right) & = & \sum_{i<j}1\left\{ \left(x_{i}-x_{j}\right)\left(y_{i}-y_{j}\right)<0\right\} \\
 & = & \sum_{i<j}1\left\{ \left(\pi\left(x_{i}\right)-\pi\left(x_{j}\right)\right)\left(\nu\left(y_{i}\right)-\nu\left(y_{j}\right)\right)<0\right\} \\
 & = & d_{K}\left(\pi,\nu\right)
\end{eqnarray*}
which depends only on the rankings $\pi$ and $\nu$ of the sample
$\left\{ x_{i}\right\} $ and $\left\{ y_{i}\right\} $. Mallows (1957)
model for rankings takes the form
\[
P_{\phi}\left(\nu\mid\pi\right)=C\left(\phi\right)e^{-\phi d_{K}\left(\pi,\nu\right)}
\]
where the normalizing constant $C\left(\phi\right)$ has a tractable
form (Fligner and Verducci, 1986) known as a Poincare polynomial (Diaconis
and Graham, 2000). Distributional forms for the data are in terms
of copulas:
\[
C_{H}\left(F\left(X\right),G\left(Y\right)\right)=H\left(X,Y\right)
\]
which are distribution functions on the unit square, having uniform
margins. Two copulas play a fundamental role in motivating the tests:
the Gaussian copula and the Frank Copula. If $\left(X,Y\right)$ has
a bivariate normal distribution $H$ with correlation $\rho$, then
its corresponding copula is
\[
C_{\rho}\left(u,v\right)=\Phi_{2}\left(\Phi^{-1}\left(u\right),\Phi^{-1}\left(v\right);\rho\right)
\]
where $\Phi$ is the standard normal CDF. The bivariate distributions
$C_{\rho}$ and $\Phi_{2}$ are indexed solely by the underlying correlation
$\rho$. The Frank copula (Frank, 1979; Genest, 1987) has the form
\[
C_{\theta}\left(u,v\right)=-\frac{1}{\theta}\log\left(1+\frac{\left(e^{-\theta u}-1\right)\left(e^{-\theta v}-1\right)}{\left(e^{-\theta}-1\right)}\right).
\]
The next two sections describe two new tests for detecting subpopulations
that support association: the Components of Spearman\textquoteright s
Footrule (CSF) test and the Components of Kendall\textquoteright s
Tau (CKT) test. The CSF test is scaled according to a Gaussian copula
and the CKT test is scaled according to a Frank copula. The CSF test
is computationally fast, and the CKT test adapts to a large variety
of alternatives. The following two sections cover their performance
under simulations. Concluding remarks are in last section.

\section*{Components of Spearman\textquoteright s Footrule (CSF)}

While Spearman\textquoteright s footrule (Diaconis and Graham, 1977)
measures the overall disarray in a sample, the distribution of individual
absolute rank differences
\[
d_{i}=\left|\pi\left(x_{i}\right)-\nu\left(y_{i}\right)\right|
\]
proves to be very useful in detecting subsamples with distinctly less
disarray than would be expected under homogeneous association. Because
the rankings depend on the whole sample, the $\left\{ d_{i}\right\} $
are not independent. Nevertheless, we loosely define their \emph{empirical
distribution} as
\[
S_{n}\left(d\right)=\frac{1}{n}\sum_{i=1}^{n}1\left\{ d_{i}\leq d\right\} .
\]
As a step toward determining asymptotic forms for this distribution,
we offer the following lemmas.
\begin{lem}
For any sample $\left\{ \left(X_{i},Y_{i}\right)\mid i=1,\ldots,n\right\} $,
from a joint distribution $H$ with compact support, let (X,Y) be
a newly, independent sampled observation. Then, for rankings $\pi$
and $\nu$ for the extended sample of $n+1$ observations, 
\[
\left[\frac{\pi\left(X\right)}{n+1},\frac{\nu\left(Y\right)}{n+1}\right]\underset{a.s.}{\rightarrow}\left[F\left(X\right),G\left(Y\right)\right]
\]
and its asymptotic distribution is the underlying copula $C_{H}\left[F\left(X\right),G\left(Y\right)\right]$
of $H$. 
\end{lem}

\begin{lem}
Under independence, the asymptotic distribution of the scaled absolute
rank differences
\[
S_{n}=\left|\frac{\pi\left(X\right)}{n+1}-\frac{\nu\left(Y\right)}{n+1}\right|
\]
is $\textrm{Beta\ensuremath{\left(1,2\right)}}$.\end{lem}
\begin{prop}
Under a Gaussian($\rho$) copula, $S_{n}$ converges to a $\textrm{Beta\ensuremath{\left(1,\beta\left(\rho\right)\right)}}$
distribution.
\end{prop}
Although we do not have a formal proof for this proposition, many
simulations with $n=1000$ affirm the proposition and produce a smooth
curve for $\beta\left(\rho\right)$. See Figure 1 for one such example.

\begin{figure}
\includegraphics[scale=0.3]{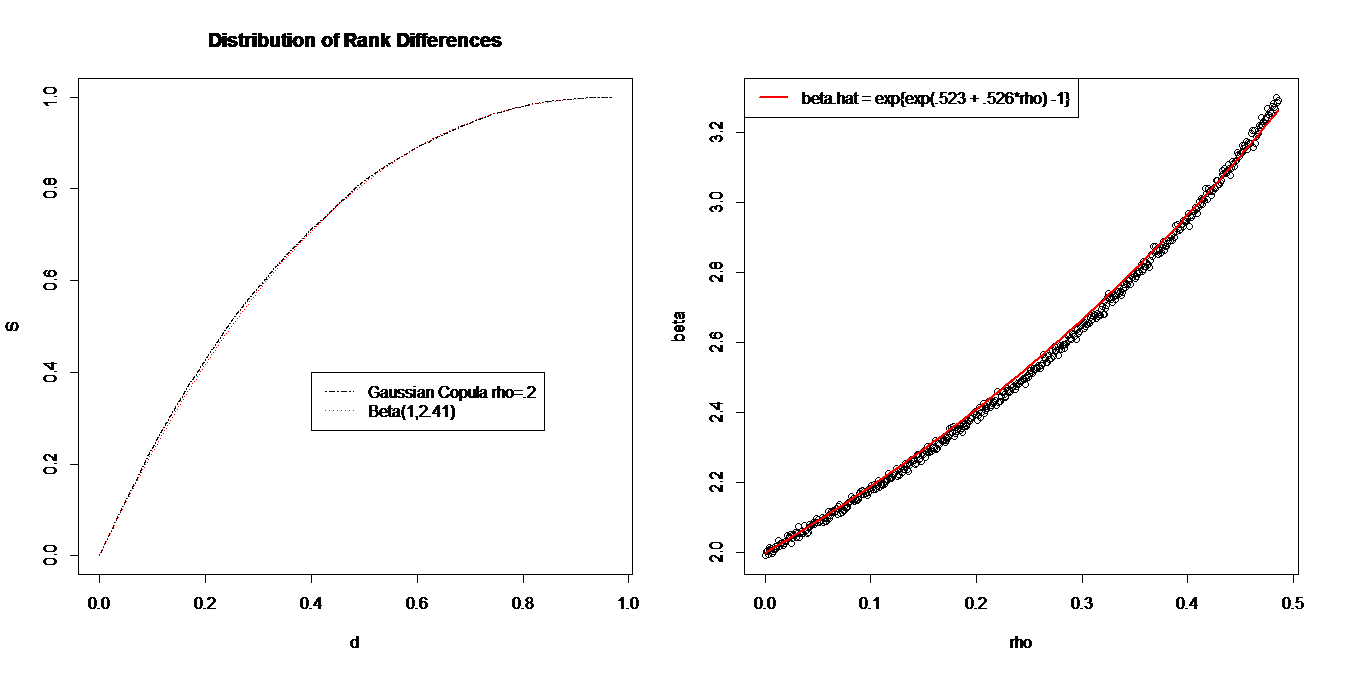}

\protect\caption{Left panel shows the close agreement of $S_{n}$ with $\textrm{Beta\ensuremath{\left(1,\beta\left(\rho\right)\right)}}$
when sampling $n=10,000$ observations from a Gaussian($\rho=0.2$)
copula. The right panel illustrates the $\mathbf{\beta\left(\rho\right)}$
curve for $0<\rho<0.5$.}
\end{figure}

The null hypothesis is that $\left(X,Y\right)$ have a Gaussian copula.
The alternative is that $\left(X,Y\right)$ come from a mixture of
two subpopulations in which under one they are independent, and under
the other they are positively associated. To test for negative association,
simply replace $Y$ by $-Y$. No particular form is assumed for the
positively associated subpopulation, but it is informative to examine
the case where this component is Gaussian. Figure 2 illustrates $S_{n}$
and its histogram under such a mixture.

\begin{figure}
\includegraphics[scale=0.35]{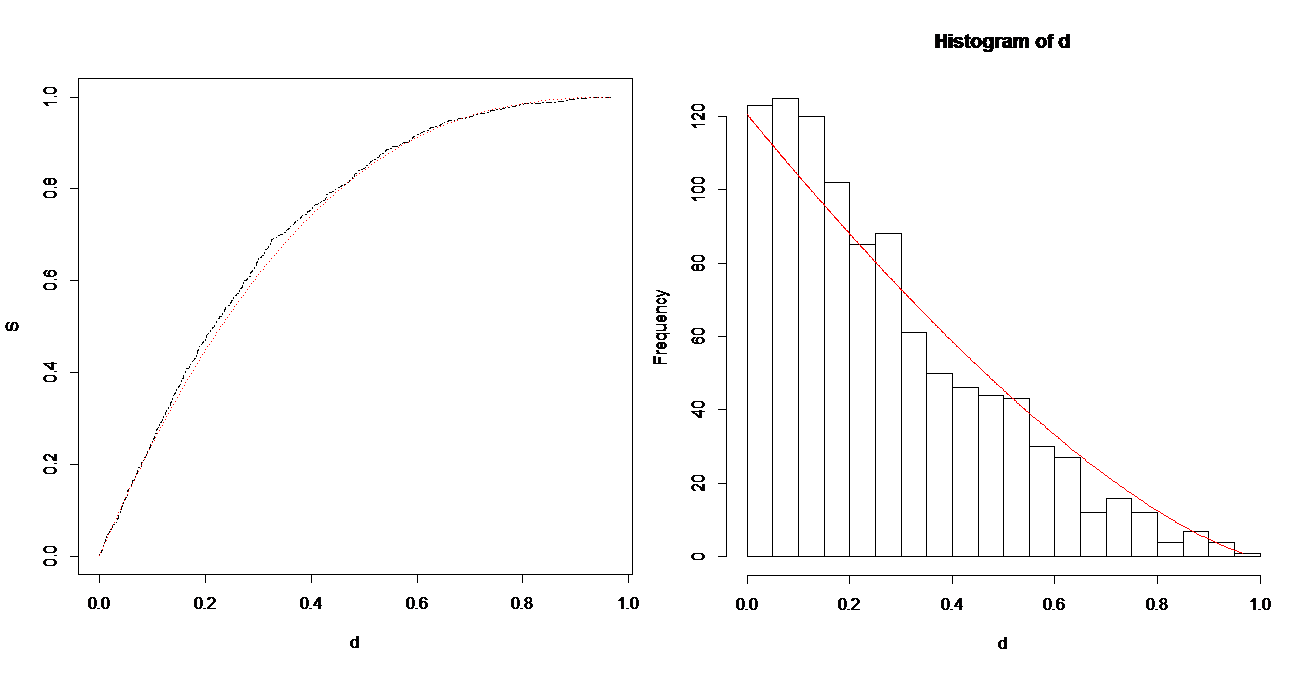}

\protect\caption{(Standardized) Distribution of Rank Differences Under Mixture of Gaussian(0.6)
and Independent Copulas with overall correlation 0.3 compared to Beta(1,2.65). }

\end{figure}

Because the differences in distributions under the null and alternative
are small, large samples are required to distinguish the two. As noted
from the histogram in Figure 2, most of the distinguishing information
is contained at the low end of the distribution. This makes sense
because a subpopulation supporting positive association should have
a surplus of points where the ranks of $X$ and $Y$ closely agree.
Thus a test statistic based on absolute ranked differences should
emphasize the lower order statistics. Such statistics come under the
heading of L-statistics. It is possible to tailor a test toward alternative
features of interest such as proportionate size of the subpopulation
and the strength of association within it. Exact distributions of
partial or weighted sums of absolute rank differences are quite complicated
due to dependencies (Sen, et al. 2011), even under the null hypothesis
of independence. A very simple general purpose test statistic is 
\[
T_{S}=\sum_{i=1}^{n}1\left\{ \frac{\left|\pi\left(x_{i}\right)-\nu\left(y_{i}\right)\right|}{n}<0.2\right\} .
\]
Using the observed overall correlation \emph{r} in place of $\rho$,
the null distribution of $T_{S}$ may be simulated under the Gaussian
copula or approximated as a Binomial test statistic using the probability
from the $\textrm{Beta}(1,\beta)$ as in Proposition 3. In the later
case, ignoring weak dependencies, the .05 level test has power of
80\% of detecting a Gaussian subpopulation of 25\% with \emph{r} =.8
for \emph{n} = 1000.

\section*{Components of Kendall\textquoteright s Tau (CKT)}

Although the CSF test is both simple and computationally efficient,
it has a conceptual shortcoming arising from the use of Spearman\textquoteright s
footrule distance to characterize association in a subpopulation.
The issue is that the components of the footrule distance in the subpopulation
depend on the encompassing population; that is, when the sample is
a full population, with associated subpopulation $\Omega$, the component
set from the footrule from $\Omega$
\[
\left\{ d_{i}\mid i\in\Omega\right\} =\left\{ \left|\pi\left(x_{i}\right)-\nu\left(y_{i}\right)\right|\mid i\in\Omega\right\} 
\]
depends heavily on the rankings $\pi$ and $\nu$ determined by the
full population. In contrast, the component set from Kendall\textquoteright s
distance depends only on the relative rankings within $\Omega$, which
may be constructed from just on the original values in $\Omega$.
That is,
\begin{eqnarray*}
\left\{ 1\left[\left(\pi\left(x_{i}\right)-\pi\left(x_{j}\right)\right)\left(\nu\left(y_{i}\right)-\nu\left(y_{j}\right)\right)<0\right]\mid i,j\in\Omega\right\} \\
=\left\{ 1\left[\left(x_{i}-x_{j}\right)\left(y_{i}-y_{j}\right)<0\right]\mid i,j\in\Omega\right\} 
\end{eqnarray*}
Thus the subpopulation discordances (components of Kendall\textquoteright s
distance) do not depend upon the embedding population, whereas the
subpopulation disarray (components of Spearman\textquoteright s footrule
distance) do. This invariance has a number of beneficial properties,
such as allowing the CKT test to retain power in situations where
the ranges of the $\left\{ X_{i}\right\} $ and $\left\{ Y_{i}\right\} $
values in the subpopulation are more restricted than those in the
full population. 

The notion of homogeneous association based on Kendall\textquoteright s
distance differs from that based the Spearman\textquoteright s footrule
used for the CSF test. In this case the natural null hypothesis should
be a distribution depending only on Kendall\textquoteright s distance.
Furthermore it should have the greatest entropy for a given value
of Kendall\textquoteright s tau because this formulation would attribute
as much variability as possible to the null distribution, making it
a conservative (least favorable) test (Lehmann and Romano, 2006).
To construct a distribution that has this structure, simply sample
from an arbitrary copula, and then reorder the \emph{Y}-values according
to a permutation $\nu\left(\mathbf{Y}\right)$ sampled independently
from a Mallows model centered at the ranking $\pi\left(\mathbf{X}\right)$
of the \emph{X}-values. Quite remarkably, any such process asymptotically
leads to a Frank copula. Proposition 4, based on Starr (2009), gives
a precise statement. 
\begin{prop}
Let $\left\{ \left(X_{i},Y_{i}\right)\mid i=1,\ldots,n\right\} $
be independent samples from a distribution $H$ with continuous marginals
$F$ and $G$, and associated copula $C$ with continuous partial
derivatives. Let $\pi\left(\mathbf{X}\right)$ be the ranking of $\pi\left(\mathbf{X}\right)=\left[X_{1},\ldots,X_{n}\right]$
and $\nu\left(\mathbf{Y}\right)$ be the ranking of $\nu\left(\mathbf{Y}\right)=\left[Y_{1},\ldots,Y_{n}\right]$
. Assume that for all $n$ sufficiently large, the conditional distribution
of $\nu\left(\mathbf{Y}\right)$ given $\pi\left(\mathbf{X}\right)$
is Mallows, with center at $\pi\left(\mathbf{X}\right)$ and scale
$\phi_{n}$. If $\phi_{n}\rightarrow0$, and there exists $\theta\neq0$
such that
\[
n\left(1-e^{-\phi_{n}}\right)\rightarrow\theta,
\]
then $C$ is the Frank Copula $C_{\theta}$.\end{prop}
\begin{proof}
First, we establish that if the conditional distribution of $\nu\left(\mathbf{Y}\right)$
given $\pi\left(\mathbf{X}\right)$ is a Mallows distribution, then
the copula C is radially symmetric. The \emph{pseudo-observations}
for each pair $\left(X_{i},Y_{i}\right)$ are defined as functions
of the pair and the empirical margins
\begin{eqnarray*}
\left(\hat{U}_{i},\hat{V}_{i}\right) & = & \frac{n}{n+1}\left(\hat{F}_{n}\left(X_{i}\right),\hat{G}_{n}\left(Y_{i}\right)\right).
\end{eqnarray*}
These are functions of the rankings $\pi\left(\mathbf{X}\right)$
and $\nu\left(\mathbf{Y}\right)$ :
\begin{eqnarray*}
\hat{U}_{i}=1-\frac{\pi\left(X_{i}\right)}{n+1}, &  & \hat{V}_{i}=1-\frac{\nu\left(Y_{i}\right)}{n+1}.
\end{eqnarray*}
By the symmetry of the Mallows model, the joint distribution of the
pseudo-observations $\left(\hat{U}_{1},\hat{V}_{1}\right),\ldots,\left(\hat{U}_{n},\hat{V}_{n}\right)$
is identical to the joint distribution of\\
 $\left(1-\hat{U}_{1},1-\hat{V}_{1}\right),\ldots,\left(1-\hat{U}_{n},1-\hat{V}_{n}\right)$
. Consider empirical distributions based on these observations (Genest
and Nešlehová, 2014):
\begin{eqnarray*}
\hat{C}_{n}\left(u,v\right) & = & \frac{1}{n}\sum_{i=1}^{n}1\left\{ \hat{U}_{i}\leq u,\hat{V}_{i}\leq v\right\} \\
\hat{D}_{n}\left(u,v\right) & = & \frac{1}{n}\sum_{i=1}^{n}1\left\{ 1-\hat{U}_{i}\leq u,1-\hat{V}_{i}\leq v\right\} .
\end{eqnarray*}
Since $H$ has continuous marginals and $C$ has continuous partial
derivatives, then Fermanian et al. (2004) established that $\hat{C}_{n}$
is a consistent estimator of the copula $C$, and likewise $\hat{D}_{n}$
is a consistent estimator of the survival copula $\bar{C}$, where
\begin{eqnarray*}
\bar{C}\left(u,v\right) & = & u+v-1+C\left(1-u,1-v\right).
\end{eqnarray*}
Hence, $\bar{C}=C$, which implies that the copula $C$ is radially
symmetric (Nelsen 2006, pg. 37). Since $C$ is radially symmetric,
an asymptotically equivalent definition of the \emph{empirical copula}
is
\begin{eqnarray*}
\tilde{C}_{n}\left(u,v\right) & = & \frac{1}{n}\sum_{i=1}^{n}1\left\{ \frac{\pi\left(X_{i}\right)}{n}\leq u,\frac{\nu\left(Y_{i}\right)}{n}\leq v\right\} \\
 & = & \frac{1}{n}\sum_{i=1}^{n}\delta_{\left(\pi\left(X_{i}\right)/n,\nu\left(Y_{i}\right)/n\right)}
\end{eqnarray*}
which places mass of $\frac{1}{n}$ on each random point $\left(\frac{\pi\left(X_{i}\right)}{n},\frac{\nu\left(Y_{i}\right)}{n}\right)\in\left[0,1\right]^{2}$.
This empirical copula is expressed by the following point process
(Starr, 2009): For $n\in\mathbb{N}$,
\begin{eqnarray*}
\mu_{n}\left(B,\omega\right) & = & \frac{1}{n}\sum_{i=1}^{n}1\left\{ \left(\frac{\pi\left(X_{i}\right)}{n},\frac{\nu\left(Y_{i}\right)}{n}\right)\in B\right\} 
\end{eqnarray*}
for each bounded Borel set $B\subseteq\mathbb{R}^{2}$.

By assumption, the regularity conditions on the Mallows scale are
satisfied as $n\rightarrow\infty$:
\begin{eqnarray*}
\phi_{n}\rightarrow0, &  & \exists\theta\in\mathbb{R}/\left\{ 0\right\} \ni n\left(1-e^{-\phi_{n}}\right)\rightarrow\theta.
\end{eqnarray*}
Under these conditions, the primary result of Starr (2009) is applied:
As $n\rightarrow\infty$, the random measures $\mu_{n}\left(\cdot,\omega\right)$
weakly converge to the measure $\mu_{\theta}$, defined by
\begin{eqnarray*}
d\mu_{\theta}\left(u,v\right) & = & \frac{\left(\theta/2\right)\sinh\left(\theta/2\right)}{\left(e^{\theta/4}\cosh\left(\theta\left[u-v\right]/2\right)-e^{-\theta/4}\cosh\left(\theta\left[u+v-1\right]/2\right)\right)^{2}}I_{[0,1]^{2}}\left(u,v\right)\partial u\partial v.
\end{eqnarray*}
Simply converting the trigonometric functions to exponential form
and simplifying yields
\begin{eqnarray*}
d\mu_{\theta}\left(u,v\right) & = & \frac{\theta\left(1-e^{-\theta}\right)e^{-\theta\left(u+v\right)}}{\left(1-e^{-\theta}-\left(1-e^{-\theta u}\right)\left(1-e^{-\theta v}\right)\right)^{2}}I_{[0,1]^{2}}\left(u,v\right)\partial u\partial v.
\end{eqnarray*}
By recognition, the limiting measure $d\mu_{\theta}$ is that of the
(Frank) Copula $C_{\theta}$. Recall, $\tilde{C}_{n}$ is a consistent
estimator of the underlying copula $C$, and converges weakly to $C_{\theta}$,
so we conclude that $C=C_{\theta}$.
\end{proof}
Pursuing this result further allows for inspection of the adequacy
of the asymptotic result for finite samples. A function $\phi\left(\theta\right)$
for matching the Mallows $\phi$ parameter to the Frank $\theta$
parameter may be obtained by equating expressions for $\tau_{\phi}$
and $\tau_{\theta}$ from these models. For any Archimedian copula,
there is a relatively simple formula $\tau=4\textrm{E}\left[C\left(U,V\right)-1\right]$
(MacKay and Genest, 1986); for the Frank copula, a specialized form
(Nelsen, 2006, p. 171; Genest, 1987) is
\begin{eqnarray*}
\tau_{\theta} & = & 1-\frac{4}{\theta}\left[1-D\left(\theta\right)\right]=1-\frac{4}{\theta}\left[1-\frac{1}{\theta}\int_{0}^{\theta}\frac{t}{e^{t}-1}\partial t\right]
\end{eqnarray*}
where the scaled integral $D\left(\gamma\right)$ is known as the
Debye-1 function, available in the \textquotedblleft gsl\textquotedblright{}
(Gnu Scientific Library) package of R. For the Mallows model,
\begin{eqnarray*}
\tau_{\phi} & = & \frac{2}{\pi}\arctan\left(.18n\phi\right)
\end{eqnarray*}
Equating $\tau_{\theta}$ and $\tau_{\phi}$ leads to the relationship
\begin{eqnarray*}
\phi & = & \frac{100}{18n}\tan\left[\frac{\pi}{2}\left\{ 1-\frac{4}{\theta}\left[1-D\left(\theta\right)\right]\right\} \right]\approx\frac{.9694}{n}\theta.
\end{eqnarray*}
\begin{figure}
\includegraphics[scale=0.25]{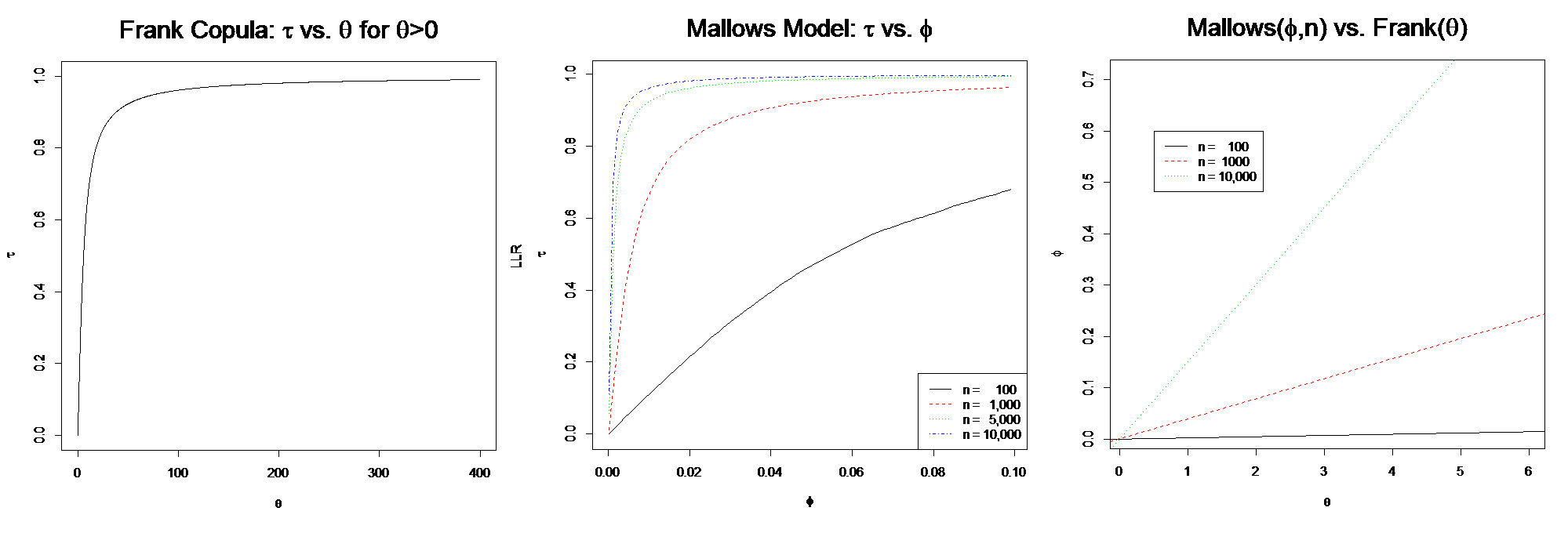}

\protect\caption{Scale relationships. Left: Kendall's $\tau$ vs. the Frank scale $\theta$;
Center: Kendall's $\tau$ vs. the Mallows scale $\phi$ for $n=100,1000,5000,10000$;
Right: Mallows $\phi_{n}$ vs. Frank$\theta$, for $n=100,1000,5000$. }
\end{figure}
Empirical evidence for the applicability of Proposition 4 comes in
two stages: 1) The distribution of Kendall\textquoteright s Distance
under $\textrm{Frank\ensuremath{\left(\theta\right)}}$ and under
$\textrm{Mallows\ensuremath{\left(\phi\left(\theta\right)\right)}}$
both converge to the same normal distribution; 2) As $n$ gets large
the product density of the sample under $\textrm{Frank\ensuremath{\left(\theta\right)}}$
converges to an increasing function of the Kendall\textquoteright s
Distance between $\pi\left(\mathbf{X}\right)$ and $\nu\left(\mathbf{Y}\right)$
of the sample. Figure 4 illustrates results from the following confirmatory
experiment: 
\begin{itemize}
\item Generate 1000 sets of 1000 points from a $\textrm{Frank\ensuremath{\left(\theta=3\right)}}$
copula 
\item Compute the Kendall distance $D$ and the Frank density $d$ for each
set 
\item Plot $d$ vs $D$ on a log-log scale 
\end{itemize}
\begin{figure}
\includegraphics[scale=0.35]{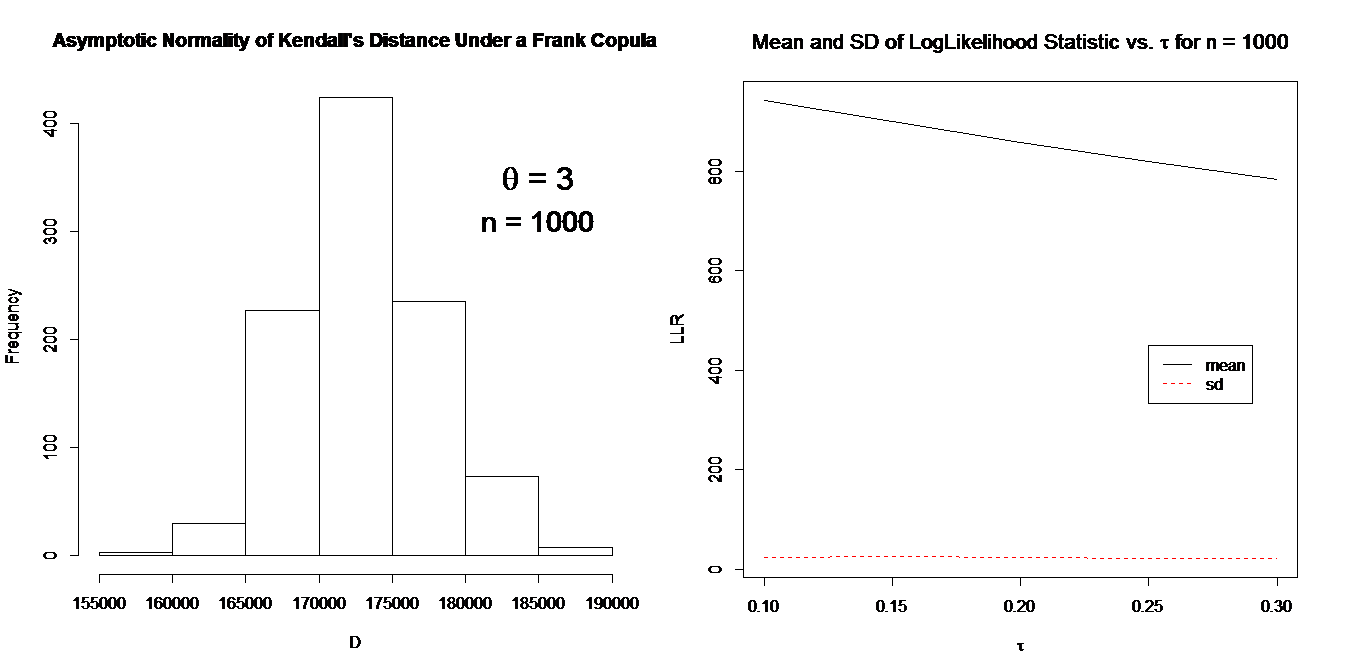}

\protect\caption{Association between $X$ and $Y$ in a Frank copula approaches a Mallows
model for the ranking $\nu\left(\mathbf{Y}\right)$ in a large sample
centered at $\pi\left(\mathbf{X}\right)$. Left: Approximate normality
Kendall's Distance $D$ in samples of size $n=1000$ from a $\textrm{Frank\ensuremath{\left(\theta=3\right)}}$
copula; Right: log-log plot of the density of each sample vs. its
Kendall's distance $D$.}

\end{figure}

Note also that the Frank copula is radially symmetric, $C\left(u,v\right)=u+v-1+C\left(1-u,1-v\right)$,
which is a necessary condition for the density of a sample to depend
only on its Kendall distance. With the assurance that there are copulas
with the conditional distribution of $\nu\left(\mathbf{Y}\right)$
given $\pi\left(\mathbf{X}\right)$ well approximated by a Mallows
model, this becomes the null hypothesis:
\begin{eqnarray*}
H_{0}:\nu\left(\mathbf{Y}\right)\circ\pi^{-1}\left(\mathbf{X}\right)\sim\textrm{Mallows}\left(\theta\right),\textrm{ for some }\theta>0.
\end{eqnarray*}
The general alternative against which we would like a test to be sensitive
is that there is a subpopulation with high association with the remainder
having (little or) no association. The test for heterogeneity should
maintain power over a wide variety of alternative distributions for
the subpopulation supporting strong association. With these considerations,
the alternative hypothesis is formulated as
\begin{eqnarray*}
H_{A}:\left(F\left(X\right),G\left(Y\right)\right)\sim M
\end{eqnarray*}
where $M$ is a mixture of two distributions: $H_{1}$ on which $\left(F\left(X\right),G\left(Y\right)\right)$
are independent and $H_{2}$ on which $\left(F\left(X\right),G\left(Y\right)\right)$
have $\tau>0$.

To test against such a general alternative, an adaptive model encompassing
the Mallows model is adopted, with the number of free parameters in
the model increasing with sample size. This components of Kendall\textquoteright s
tau (CKT) test proceeds in four steps:
\begin{enumerate}
\item Fit a Mallows model centered at $\pi\left(\mathbf{X}\right)$ to $\nu\left(\mathbf{Y}\right)$
and compute the likelihood. 
\item Reorder the data points $\left\{ \left(X_{i},Y_{i}\right)\mid i=1,\ldots,n\right\} $,
so that Kendall's tau coefficient is decreasing. See Yu\emph{ et al}.
(2011). Call the reordering $\sigma$.
\item Smoothly fit a multistage ranking model to the relative rankings of\\
$\left[Y_{\sigma\left(1\right)},\ldots,Y_{\sigma\left(k\right)}\right]$
to $\left[X_{\sigma\left(1\right)},\ldots,X_{\sigma\left(k\right)}\right]$
at each stage $k$. See Sampath and Verducci (2013). Compute the likelihood
under this (encompassing) model.
\item Use the (Generalized) Likelihood Ratio statistic to test $H_{0}$.
\end{enumerate}
Comments on the four steps: 
\begin{enumerate}
\item Since Kendall\textquoteright s tau distance is invariant to reordering
of observations, this is the same as fitting a Mallows model, centered
at ranking $\left(\sigma X\right)$, to the ranking $\left(\sigma Y\right)$,
where $\sigma$ is the taupath reordering.
\item The idea of reordering is to put the points displaying the highest
amount of association earlier in the sequence in order to identify
the subpopulation with highest empirical association. The reordering
is not unique. Yu \emph{et al. }(2011) discuss various algorithms. 
\item The multistage ranking model decomposes the number of discordances
{[}up to ($n$ choose 2){]} between ranking $\left(\sigma Y\right)$
and ranking $\left(\sigma X\right)$, as a sum of $n-1$ variables
$\left\{ V_{k}\right\} $ with ranges $\left\{ 0,\ldots,k\right\} $,
$k=1,\ldots,n-1$. The model has likelihood $L=c\left(\theta\right)e^{-\sum\theta_{k}V_{k}}$
which reduces to the likelihood of Mallows model when all component
parameters are equal. 
\item The conditions needed to justify an asymptotic chi-square distribution
for this statistic do not hold in this setting. Currently, we simulate
the distribution under the Frank copula to get an appropriate reference.
We are working to find a more precise characterization of the LR in
this setting. 
\end{enumerate}
The null distribution of this likelihood ratio appears to be close
to normal, with its mean decreasing with the common correlation $\tau$,
and standard deviation constant. See Figure 5. Note that, for $n=1000$,
the variance of $2\cdot\textrm{LLR}\approx2500$ is clearly less than
its $2\cdot\textrm{mean\ensuremath{\left(2\cdot\textrm{LLR}\right)}}$
theoretical value for a chi square distribution, which is in the range
$\left(3100,3800\right)$ when $\tau\in\left(.10,.30\right)$.

\begin{figure}
\includegraphics[scale=0.35]{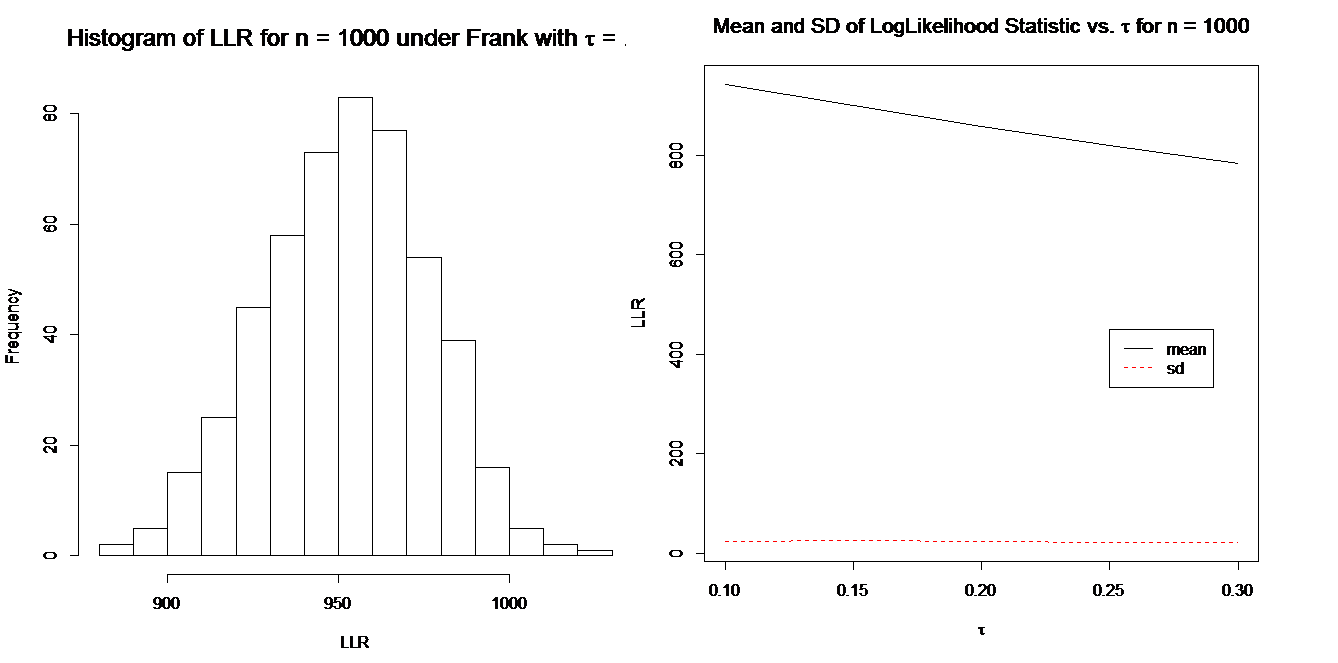}

\protect\caption{Simulation of Log Likelihood Ratio (LLR) Statistic for CKT test under
Frank Copulas. Left: Histogram of $500$ simulations of size $n=1000$.
Right: Decreasing pattern of mean and constancy of standard deviation
for LLR under Frank Copulas at different levels of $\tau$ for $n=1000$.}
\end{figure}

Instead of fixed $n$ and varying $\tau$, Figure 6 depicts the relationship
between \emph{LLR} and $n$ with fixed $\tau=0.1$. The overall relationship
between the moments of \emph{LLR} and the parameters $\tau$ and \emph{$n$}
is not yet known, but using a practical additive approximation in
the range $.1<\tau<.3$ and $500<n<3000$, the basic asymptotic \textgreek{a}-level
CKT test has the form: Reject $H_{0}$ if
\[
Z=\frac{\textrm{LLR}-\left(n+20-797\hat{\tau}\right)}{0.02n+7}>z_{1-\alpha},
\]
where $\hat{\tau}$ is Kendall\textquoteright s correlation coefficient
and $z_{1-\alpha}$ is the $\left(1-\alpha\right)^{\textrm{th}}$
quantile of the standard normal.

\section*{Simulations for Robustness and Power}

First, performance of the tests is checked by maintenance of levels
under various Gaussian and Frank copulas; subsequently power is examined.
The CSF test is based on the number of absolute rank differences less
than .2. Figure 7 shows the null distributions of p-values for the
CSF test applied to samples of size $n=1000$ generated 100,000 times
under the Gaussian($\rho$) models. These distributions start to become
stochastically smaller than uniform for $\rho>0.45$. Otherwise the
test is conservative in the range $0<\rho<0.45$ and $0<\alpha<0.05$
as illustrated by the observed number of type 1 errors at the $\alpha=0.05$
level.

\begin{figure}
\begin{centering}
\includegraphics[scale=0.35]{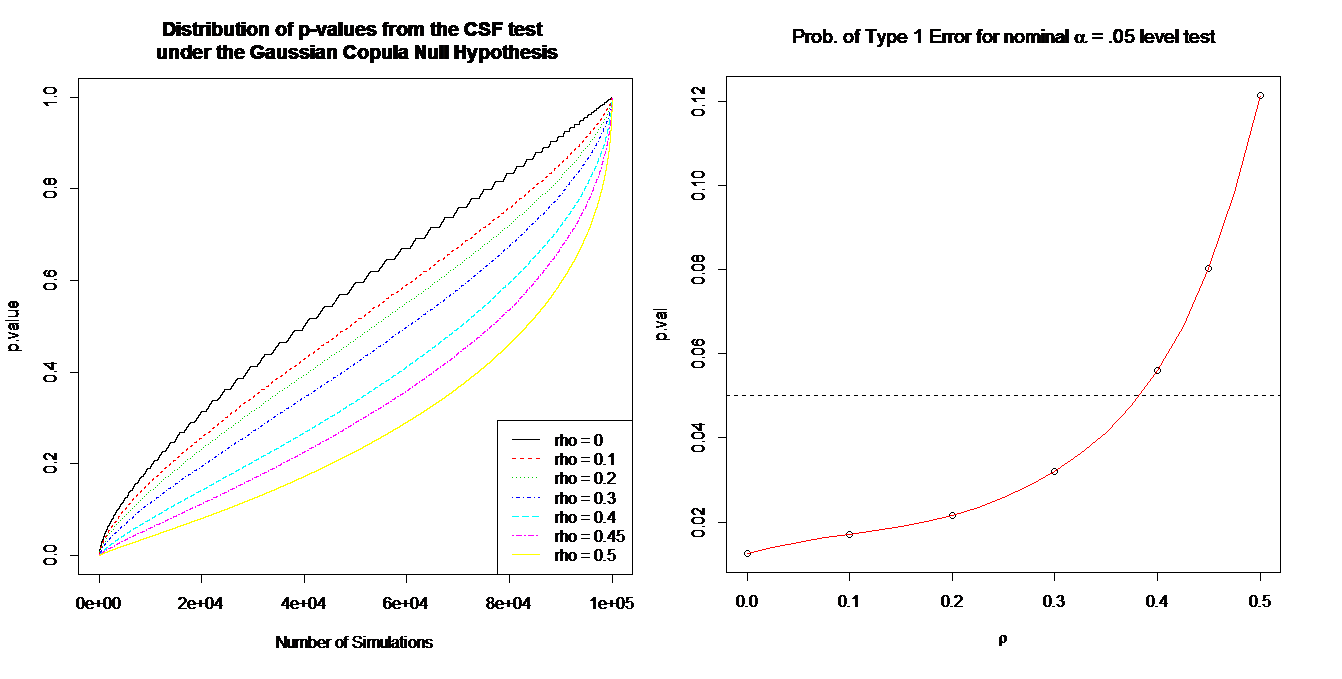}
\par\end{centering}

\protect\caption{CSF results from $100,000$ simulation experiments of size $n=1000$
for values of $\rho$ in a Gaussian copula. Left: Distribution of
p-values for 7 values of null correlation; Right: Observed probabilities
of Type I error for a nominal $\alpha=0.05$ test. The test is conservative
for values of $\rho<0.45$. }
\end{figure}

Under similar Gaussian copulas, the adjustment of the mean of the
log-likelihood for the estimated overall $\tau$ makes the CKT test
behave conservatively for large values of $\rho$, but gives highly
significant values for $r$ values near 0. See Figure 8, in which,
due to computational limitations, lowess-smoothed curves describe
the p-distribution based on only 100 simulations. In the presence
of very low overall correlation, it is advisable to use the CSF test
as a screen for the CKT, which will protect the CKT from finding uneven
levels of $\tau$ association when $\rho$ association is homogeneous.
Again, this tendency toward excess false positives happens only when
the overall $\rho$ association is close to 0. In this case a special
test (Sampath and Verducci, 2013) is available for the null hypothesis
of independence. Under a Frank copula, the CSF test behaves properly
near independence, but loses its level when \textgreek{t} gets large.
See Figure 9.

\begin{figure}
\includegraphics[scale=0.3]{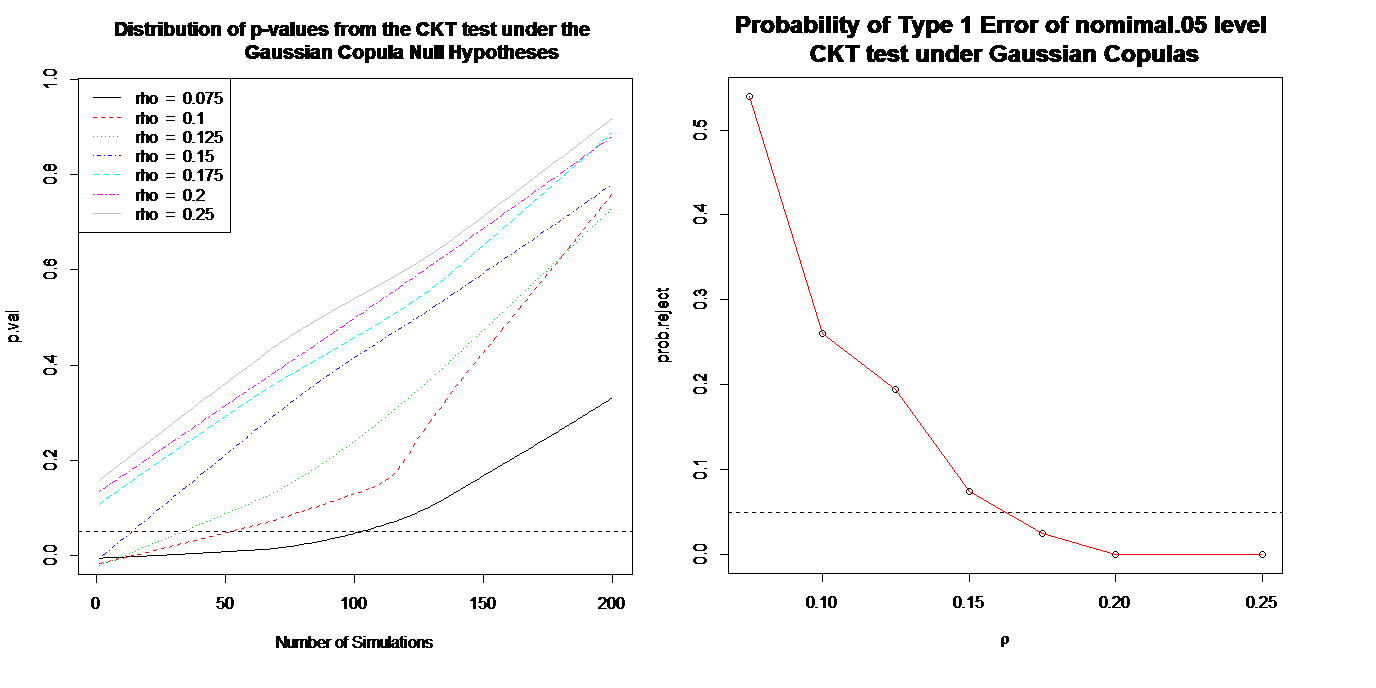}

\protect\caption{CKT results from 200 simulation experiments of size $n=1000$ for
values of $\rho$ in a Gaussian copula; Left: Distribution of p-values
for 7 values of the null correlation; Right: Obeserved probabilities
of Type I error for a nominal $\alpha=0.05$ test. The test is conservative
for values of $\rho>0.16$. }
\end{figure}

Several factors affect the power curves of both the CSF and CKT tests:
\emph{sample size} (\emph{n} is fixed at 500 or 1000); \emph{proportionate
size of the subpopulation} (fixed at 40\%); \emph{strength of association
in the subpopulation} $\left(\rho,\tau\in\left\{ .7,.8,.9\right\} \right)$;
and, most importantly, the \emph{form of the subpopulation}. Against
the null hypothesis of a Gaussian copula, the alternative is a mixture
of copulas, where the variables are assumed to be independent in the
complement of the subpopulation. Against the null hypothesis of a
Frank discordancescopula, the subpopulation is selected at random
and its conditional distribution is forced into a stronger Mallows
model. This allows the population margins to remain uniform while
possibly restricting the range of the subpopulation.

Figure 10 shows the distribution of p-values of both CSF and CKT tests
against 40\% Gaussian with $.12\leq\rho\leq.13$. For this range of
overall correlation the CKT test holds its level and is conservative
for overall correlation $\rho<.12$, which is the case here. Nevertheless,
it achieves perfect power when the subpopulation $\rho\geq.125$,
even though its power quickly diminishes to 10\% for $\rho=.12$ in
the subpopulation. It also performs better than CSF in this range.

\begin{figure}

\begin{centering}
\includegraphics[bb=0bp 160bp 700bp 840bp,clip,scale=0.4]{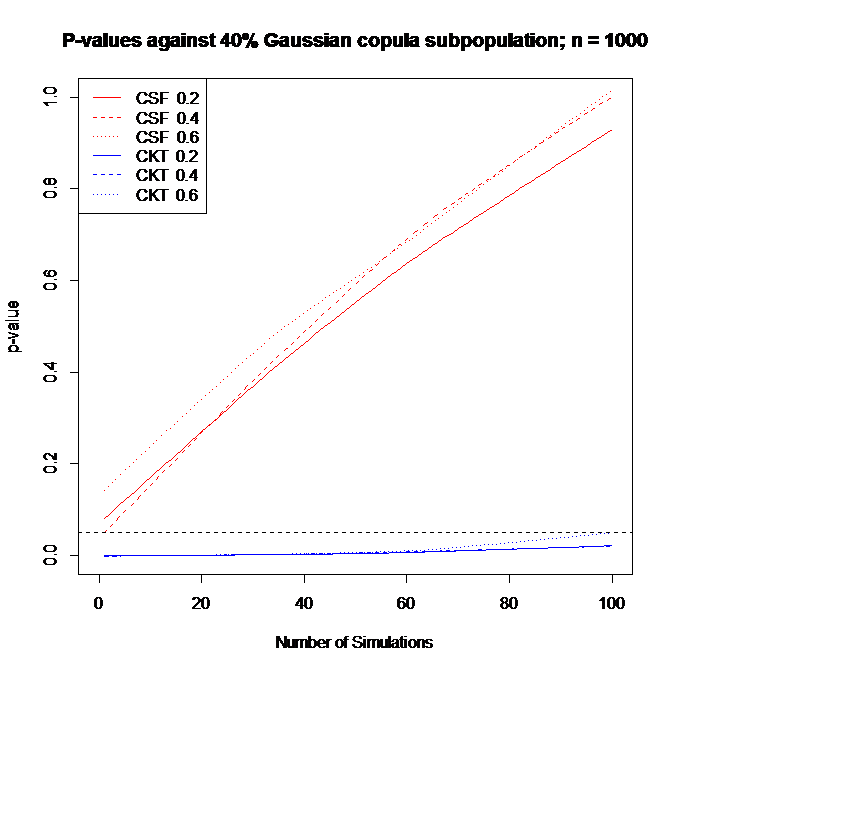}
\par\end{centering}

\protect\caption{P-values against Alternative Mixture of Gaussian copulas with sample
size $n=1000$ and subpopulation proportion $=40\%$. Dashed horizontal
lines at $\alpha=0.05$ indicates the power of a level 0.05 test.}

\end{figure}

Under the Mallows alternative, $n=1000$ points are generated from
a uniform distribution, 400 points are then sampled from a quantile
range of x values and the y values resorted according to a random
draw from a $\textrm{Mallow\ensuremath{\left(\phi\left(\tau\right)\right)}}$
model. Values of $\tau$ used are .4, .5, and .6. Figure 11 shows
the distributions of p-values from the $\alpha=0.05$ level CSF and
CKT tests over 100 simulations. The left panel corresponds to the
subpopulation being sampled from the full range, while the right panel
corresponds to samples between the 20th and 80th percentiles of x-values.
The CKT test performs much better than the CSF test against these
alternatives. The CKT has essentially perfect detection when the subpopulation
spans the whole range, and at least 70\% power in the 20-80 percentile
range. The CSF has no power in either scenario.

\begin{figure}
\begin{centering}
\includegraphics[scale=0.3]{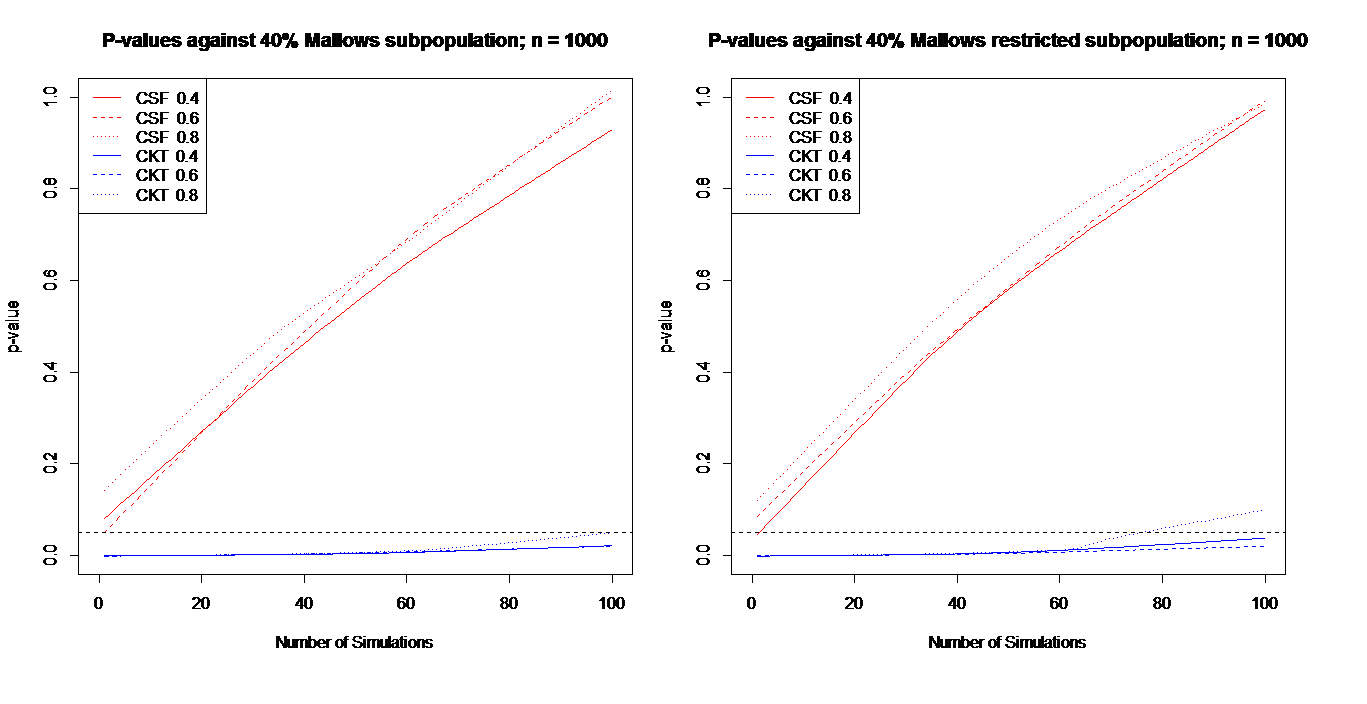}
\par\end{centering}

\protect\caption{Distribution of p-values for Mallows alternatives, based on 100 simulations
of sample size $n=1000$. Left: Associated subgroup spans the full
range of x-values. Right: Span is restricted to the 20-to-80th percentiles
region of x-values in the larger population.}

\end{figure}

\section*{Example}

Wine cultivars are varieties of grapes that have been cultivated through
selective breeding. Different varieties may be characterized by certain
chemical properties of the wine they produce. Early work in supervised
learning has been used to classify wine cultivars using chemical measurements
of wine sample (Aeberhard, et al. 1993). These data, available at
\href{http://archive.ics.uci.edu/ml/machine-learning-databases/wine}{(http://archive.ics.uci.edu/ml/machine-learning-databases/wine)},
are reanalyzed here using the CKT and CSF tests as unsupervised methods
of detecting different association structures that might help characterize
different cultivated varieties. 

Figure 12 shows the relationship between flavenoids and phenols in
the data set consisting of 13 measurements from 178 wine samples derived
from 3 different cultivars. To the untrained eye, the overall plot
looks typical of homogeneous association, but both the CKT (\emph{p}
=.0002) and CSF (\emph{p}=.027) indicate heterogeneity. Identification
of cultivars in the plot shows separation of cultivar 1 and 3 samples
from each other, with slightly negative association within each of
these groups; however, their positioning contributes a kind of ecological
correlation to the overall sample. In contrast, samples from cultivar
2 show a strong positive association between flavenoid and phenol
content. This suggests an underlying genetic difference.

It is impressive that CKT can detect this heterogeneity of association
from the unlabelled data, which looks like an overall positive association,
part of which is ecological correlation. Although the CSF test does
also indicate association, it is not as sensitive at detecting it
in this situation, and its p-value would not present a strong case
for heterogeneity if any correction is attempted for multiple comparisons
over the 13 choose 2 (78) pairs of variables available.

\section*{Concluding Remarks}

The ability to detect subpopulations that drive association has the
potential of changing the way statistics are used to unveil structures
in \textquotedblleft Big Data.\textquotedblright{} Instead of employing
extensive model searching with complex interaction, now relatively
model-free methods are available to ascertain with precision is there
is any simple mixture that better explains monotone association between
variables. The CSF and CKT tests achieve this, either working together
to screen and confirm or separately to find different forms of the
subpopulation that most strongly supports the association. 

These tests, however, are formally restricted to different forms of
the meaning of \textquotedblleft homogeneous association.\textquotedblright{}
Strict legitimacy of the CSF test depends on the assumption of a Gaussian
copula underlying the null distribution, whereas the CKT test depends
on the assumption of a Frank copula underlying the null distribution.
Although there is some evidence of limited robustness, much more work
should be done to explore the behavior of these tests under general
conditions. For example, both the Gaussian and Frank copulas are radially
symmetric; it is unclear how sensitive the tests would be to asymmetric
notions of homogeneous association. 

The computationally efficiency of the CSF test is important because
the sample size n needs to be in the thousands before there is much
hope of reliably detecting these subtle but important differences.
In contrast with the CSF test, the justification of CKT is a bit more
compelling, based on intrinsic association within the subpopulation.
We have been using CSF at a liberal $\alpha=0.05$ level as a screening
devise to reduce the number of pairs of variables to be tested at
a more stringent level. 

Detecting heterogeneity of association is a difficult task. Such detection
is practical only when the overall association is not too strong,
the association in the subpopulation is strong, and the sample size
is large. Nevertheless, such scenarios abound. We believe that these
new methods will make Statistics ever more relevant in making good
sense from Big Data.

\end{document}